\documentclass[a4paper]{article}

\usepackage[english]{babel}
\usepackage[T1]{fontenc}
\usepackage{amsmath, amsfonts, amsopn, amsthm, amssymb}
\usepackage{epsfig}
\usepackage{fancyvrb}
\usepackage{hyperref}
\usepackage{enumitem}
\usepackage[utf8]{inputenc}
\usepackage[scaled]{beramono} 
\usepackage{lmodern}

\theoremstyle{definition}
\newtheorem{prop}{Proposition}
\newtheorem{lemma}{Lemma}
\newtheorem{remark}{Remark}

\newcommand\E{\mathbb E}
\newcommand\R{\mathbb R}
\newcommand\N{\mathcal N}

\renewcommand\S{\mathcal S}

\newcommand*{\eg}{e.g.\,}
\newcommand*{\ie}{i.e.\,}

\begin{document}
\title{SHAP for additively modeled features \\
	 in a boosted trees model}
\author{Michael Mayer \\
	{\small la Mobilière, Bern, Switzerland} \\
	{\small michael.mayer@mobiliar.ch} \\
}
\date{\today}

\maketitle

\begin{abstract}
	An important technique to explore a black-box machine learning (ML) model is called SHAP (SHapley Additive exPlanation). SHAP values decompose predictions into contributions of the features in a fair way. We will show that for a boosted trees model with some or all features being additively modeled, the SHAP dependence plot of such a feature corresponds to its partial dependence plot up to a vertical shift. We illustrate the result with XGBoost.
\end{abstract}
\textbf{\textit{keywords:}} machine learning, xai, shap, gradient boosted trees
\section{Introduction}
In their 2017 article \cite{lundberg2017}, Scott Lundberg and Su-In Lee introduced a method called ``SHAP'' (SHapley Additive exPlanation) used to explain black-box machine learning (ML) models. Since then, SHAP has received a lot of attention \cite{molnar2019}. SHAP values decompose a prediction into contributions of the covariates in a fair way, using ideas from cooperative game theory. 

Originally intended as a method to explain a model locally around a specific observation, SHAP is nowadays also used to explain a model globally by decomposing many predictions and then visualizing the resulting decompositions by statistical plots, see \cite{lundberg2020, molnar2019}. An important visualization is the SHAP dependence plot of a feature $j$, which is a scatter plot of SHAP values of feature $j$ against the corresponding feature values. It serves as a rich alternative to Friedman's partial dependence plot \cite{friedman2001}. Such global model inspections work especially well for gradient boosted trees implementations like XGBoost \cite{chen2016}, LightGBM \cite{ke2017}, and CatBoost \cite{prokhorenkova2018}, because they benefit from the extremely fast TreeSHAP algorithm for tree-based models introduced in \cite{lundberg2020}.

By using tree stumps as base learners \cite{lou2012, nori2019, mayer2022} or by setting interaction constraints \cite{lee2015, mayer2022}, gradient boosted trees allow to model all or certain features additively, \ie, without interactions. Such additive or partly additive models often provide an excellent trade-off between interpretability and performance, see \cite{mayer2022} for use cases in geographic modeling. Interaction constraints are implemented, \eg, in XGBoost and LightGBM.

In this paper, we will show that the SHAP dependence plot of an additively modeled feature in a boosted trees model is identical to its partial dependence plot up to a vertical shift. This makes the two types of effect plots interchangeable in this situation.

In the next section, we will provide the necessary background and present our results. In Section~\ref{sec:illustration}, we illustrate these with the statistical software R \cite{r}.
\section{Results}
SHAP is based on Shapley values \cite{shapley1953} from cooperative game theory: Let $\N$ denote a set of $p = |\N|$ players playing a cooperative game with a numerical payoff to be maximized. The contribution to the payoff of a subset $\S \subseteq \N$ of the players is measured by a function $v: \S \mapsto \R$. Shapley's seminal article \cite{shapley1953} answers the question ``How to distribute the total payoff fairly among the players?'' as follows: The amount that player $j$ should receive is called {\em Shapley value} and equals
\begin{equation}\label{eq:shapley}
	\phi_j = \sum_{\S \subseteq \N \setminus\{j\}} \underbrace{\frac{|\S|!(p - |\S| - 1)!}{p!}}_\text{Shapley weight} \big(\underbrace{v(\S \cup \{j\}) - v(\S)}_{\text{Contribution of } j}\big), \ \ j = 1, \dots, p.
\end{equation}
Thus, a player's contribution is equal to the weighted average of his contribution to each possible coalition $\S$ of other players. ``Fairness'' is characterized by a number of properties, see \cite{shapley1953} for the original formulation or, \eg, \cite{aas2021} for a modern take.

Shapley values are not only relevant in game theory, but also in statistical modeling: For example, \cite{lipovetsky2001} used it as a strategy to decompose the R-squared of a linear regression into additive contributions of the covariates in a fair way. Their approach involved retraining the model for each of the $2^p$ feature subsets. The major breakthrough of Shapley values in statistics and ML came with the 2017 article \cite{lundberg2017} by Scott Lundberg and Su-In Lee who used Shapley values to decompose {\em predictions} into feature contributions called SHAP values. The outcome of the cooperative game is the prediction, and the features are the players, see also \cite{strumbelj2010, strumbelj2014}. 

The prediction $f(x)$ of a given feature vector $x = (x_1, \dots, x_p)' \in \R^p$ is to be decomposed into contributions $\phi_j \in \R$, $1 \le j \le p$, such that ``local accuracy'' (also called ``efficiency'') \cite{lundberg2020}
\begin{equation}\label{eq:localaccuracy}
	f(x) = \phi_o + \sum_{j = 1}^p \phi_j
\end{equation}
holds, where $\phi_o = \E_X(f(X))$. Only if the $\phi_j$ are Shapley values, the decomposition will be fair in the sense of Shapley. In order to apply Equation~\ref{eq:shapley}, a contribution function $v$ must be selected. A natural candidate is $v(\S) = f_\S(x)$, where $x_\S$ represents the components in the feature subset $\S \subseteq \N$ selected from the full feature set $\N$. 

Important algorithms to estimate $f_\S(x)$ include Kernel SHAP \cite{lundberg2017, aas2021} (model-agnostic), and TreeSHAP \cite{lundberg2020} (for tree-based models). Thanks to the latter, the predictions of many (for instance $n = 1000$) observations can be decomposed within seconds. Analyzing their SHAP values with simple descriptive statistics leads to rich global explanations of the model \cite{lundberg2020}.

Let $X$ be the $(n \times p)$ feature matrix with elements $x_{ij}$ and $\Phi$ the corresponding $(n \times p)$ matrix of SHAP values with elements $\phi_{ij}$, $1 \le i \le n$, $1 \le j \le p$. Furthermore, let $\phi_o$ be the average model prediction over all $n$ observations. Thanks to the local accuracy property~\ref{eq:localaccuracy},
$$
	f(x^{(i)}) = \phi_o + \sum_{j = 1}^{p} \phi_{ij}, \ 1 \le i \le n,
$$
where $x^{(1)}, \dots, x^{(n)}$ are the $n$ feature vectors.
In order to visualize the effect of the $j$-th feature, one then considers the so-called SHAP dependence plot representing the graph
$$
\{(x_{ij}, \phi_{ij}), 1 \le i \le n\}.
$$
This scatter plot serves as an informative alternative to the partial dependence plot \cite{friedman2001}. The empirical (one-dimensional) partial dependence function of feature $x_j$ is given by
$$
  \text{PD}_j(v) = \frac{1}{n} \sum_{i = 1}^n f(v, x_{\setminus j}^{(i)}),
$$
where $x_{\setminus j}^{(i)}$ denotes all but the $j$-th component of $x^{(i)}$. The partial dependence plot (PDP) represents the graph 
$$
	\{(v_k, \text{PD}_j(v_k)), 1 \le k \le K\}
$$ 
for a grid of values $v_k$. It shows how $f$ reacts on changes in $x_j$ (averaged over potential interaction effects), while keeping everything else fixed (``ceteris paribus''). For models with additive $x_j$, the PDP is parallel to any individual conditional expectation (ICE) curve \cite{goldstein2015}, and gives a complete description of the effect of $x_j$ \cite{friedman2001}. Interestingly, Friedman describes in \cite{friedman2001} (page 27) an algorithm to calculate the PDP of a decision tree for the full training data using an efficient method {\em identical} to the (path-dependent) TreeSHAP Algorithm 1 \cite{lundberg2020}, a fact that was also pointed out in an online comment \cite{hug2019}. (Since our result is not limited to path-dependent TreeSHAP, we do not make use of this fact.)

In order to proof our main result, we provide a simple Lemma on SHAP values of univariable models.
\begin{lemma}\label{lem}
Let $g$ be a model with single feature $z \in \R$ and partial dependence function $z \mapsto \text{PD}(z) = g(z)$. Furthermore, let $z^{(i)}$, $1 \le i \le n$, denote $n$ observations of $z$ with SHAP values $\phi^{(i)}$ satisfying local accuracy~\ref{eq:localaccuracy}. Then, for all $1 \le i \le n$,
$$
	\phi^{(i)} = g(z^{(i)}) - \phi_o = \text{PD}(z^{(i)}) - \phi_o,
$$
for a constant $\phi_o$, typically (but not necessarily) $\phi_o = \frac{1}{n} \sum_{i = 1}^n g(z^{(i)})$.
\end{lemma}
\begin{proof}
	From local accuracy, $g(z^{(i)}) = \phi_o + \phi^{(i)}$, $1 \le i \le n$, for some constant $\phi_o$. Thus, $\phi^{(i)} = g(z^{(i)}) - \phi_o$. Plugging $z = z^{(i)}$ into $\text{PD}(z) = g(z)$ finishes the proof. 
\end{proof}
From Lemma~\ref{lem}, the main result follows:
\begin{prop}\label{prop}
	Let $f$ be a boosted trees model. Denote by $x_{ij}$ the $i$-th value of the $j$-th feature in a sample of size $n$, $1 \le i \le n$, $1\le j \le p$, and by $\phi_{ij}$ the corresponding SHAP values obtained from the (path-dependent or interventional) TreeSHAP algorithm in  \cite{lundberg2020}. If the model is additive in a specific feature $x_{j'}$, then there exists $c \in \R$ such that $\phi_{ij'} = \text{PD}_{j'}(x_{ij'}) + c$ for all $1\le i \le n$, where the partial dependence function $\text{PD}_{j'}$ is calculated from arbitrary feature vectors.
\end{prop}
\begin{proof}
	Let $f$ be the sum of $K\ge 1$ decision trees $f_1, \dots, f_K$. Since $f$ is additive in $x_{j'}$, some trees split only on $x_{j'}$, and the others do not split on $x_{j'}$. Let $M_1 \subset \{1, \dots, K\}$ denote the indices of the first group, and $M_2 = \{1, \dots, K\} \setminus M_1$ the others. Denote by $\phi^{(k)}_{ij'}$ the SHAP values of the $k$ tree for the $j'$-th feature, $1 \le i \le n$.
	Let $k \in M_2$. From the missingness (or dummy)  property satisfied by the TreeSHAP algorithms in \cite{lundberg2020}, $\phi^{(k)}_{ij'} = 0$. Furthermore, the partial dependence function does not depend on $v$, \ie, $\text{PD}_{j'}(v) = c_k$ for some constant $c_k$. Thus, $\phi_{ij'} = \text{PD}_{j'}(x_{ij'}) - c_k$, and the claim holds for trees in $M_2$.
	Now, let $k \in M_1$. The prediction function $f_k$ only depends on the component $j'$ and thus we can define the univariable function $f^*_k(x_{j'}) := f_k(x)$. Applying Lemma~\ref{lem} with $g = f^*_k$ shows the claim for trees in $M_1$. By additivity, since the claim holds for all trees in $M_1$ and $M_2$, the claim of Proposition~\ref{prop} follows. 
\end{proof}
\begin{remark}[Practical consequences]
\begin{itemize}
	\item[]
	\item Up to a vertical shift, the SHAP dependence plot corresponds to the PDP (and any ICE curve) evaluated at unique values of $x_{ij'}$. 
	\item Just like the PDP, the SHAP dependence plot can be interpreted ceteris paribus. Thus, it gives a complete description of the effect of $x_{j'}$.
	\item Two observations with equal value in the $j'$-th feature will receive the same SHAP value for that feature.
	\item The result holds for models with any number of additive features.
\end{itemize}
\end{remark}
\begin{remark}[Case weights]
	It is easy to see that the proof of Proposition~\ref{prop} also works in modeling situation involving case weights, just with different constants.
\end{remark}
\begin{remark}[Link functions]
	For models involving a link function, some implementations of TreeSHAP provide SHAP values on the link scale. In such a case, the correspondence with the PDP will occur on the link scale, see Section~\ref{sec:illustration} for an example with a log link.
\end{remark}
\begin{remark}[Monotonicity]
	In addition to using interaction constraints or limiting the maximum tree depth to 1, most boosted trees implementations allow to force some effects to be monotonically increasing or decreasing. This provides another way to structure a model. For additive features, thanks to the ceteris paribus interpretation implied by Proposition~\ref{prop}, monotonicity is visible not only in the PDP, but also in the SHAP dependence plot.
\end{remark}
\section{Illustration}\label{sec:illustration}
To illustrate the main result, we use a subset of an insurance dataset from OpenML (ID 42876) \url{https://www.openml.org/d/42876} synthetically generated by Colin Priest. It describes workers compensation claims regarding their ultimate loss, the initial claim amount, and other information. 

We will first describe all relevant variables and then study the result of two boosted trees models for the ultimate loss. The first model is additive in all covariates, while the second model is additive only in some of the features.
\subsection{Data}
The prepared data contains 82,017 observations of the ultimate claim amount in USD (= the model response \texttt{Ultimate}) and the following ten features:
\begin{itemize}[itemsep=0pt]
	\item \texttt{LogInitial}: Logarithmic initial case reserve in USD
	\item \texttt{LogWeeklyPay}: Logarithmic salary per week in USD
	\item \texttt{LogDelay}: Logarithmic number of days from accident to reporting
	\item \texttt{LogAge}: Logarithm of age of worker
	\item \texttt{DateNum}: Decimal year (1988 - 2007)
	\item \texttt{PartTime}: Dummy (1: part time worker, 0: full time worker)
	\item \texttt{Female}: Dummy (1: Female, 0: Other)
	\item \texttt{Married}: Dummy (1: Married, 0: Other)	
	\item \texttt{WeekDay}: Weekday (1: Monday, 7: Sunday)	
	\item \texttt{Hour}: Hour of accident (0-23)					
\end{itemize}
Example rows of the prepared data look as follows:
\begin{Verbatim}[fontsize=\small, frame=single]
Ultimate LogInitial LogWeeklyPay LogDelay LogAge DateNum
    102.       9.16         6.21     2.48   3.81   2005.
   1451        8.01         5.65     3.00   3.69   1995.
    320.       6.91         6.25     3.69   3.91   2002.
    108        4.70         5.30     2.94   2.94   1995.
   7111.       9.18         6.64     3.22   2.94   2005.
   8379.       9.16         5.30     4.19   3.04   2002.
   
PartTime Female Married WeekDay  Hour
       0      0       0       4     9
       0      0       1       3    15
       0      1       0       4     7
       0      0       0       4    14
       0      0       0       5    14
       0      1       0       3    12
\end{Verbatim}
The response and the features are summarized by Figures~\ref{fig:response}--\ref{fig:bar}.
\begin{figure}
	\centering
	\includegraphics[width=0.95\textwidth]{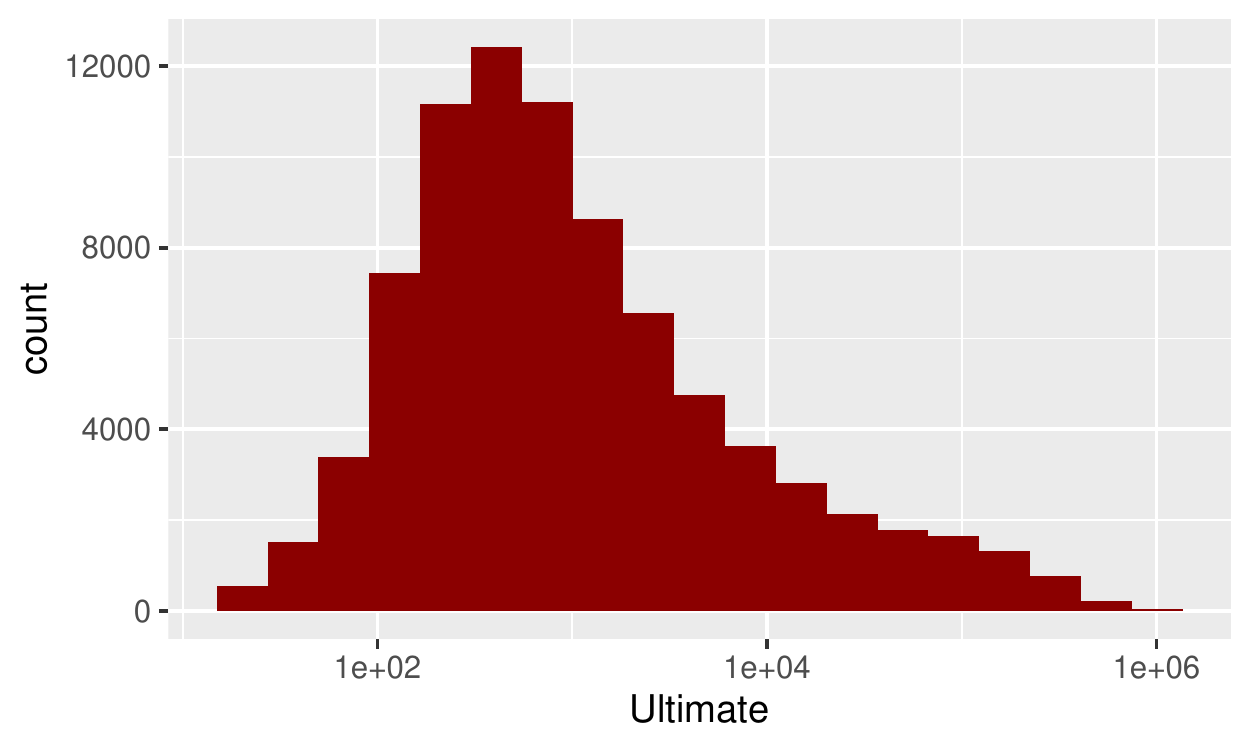}
	\caption{Histogram of (logarithmic) response.}
	\label{fig:response}
\end{figure}
\begin{figure}
	\centering
	\includegraphics[width=0.99\textwidth]{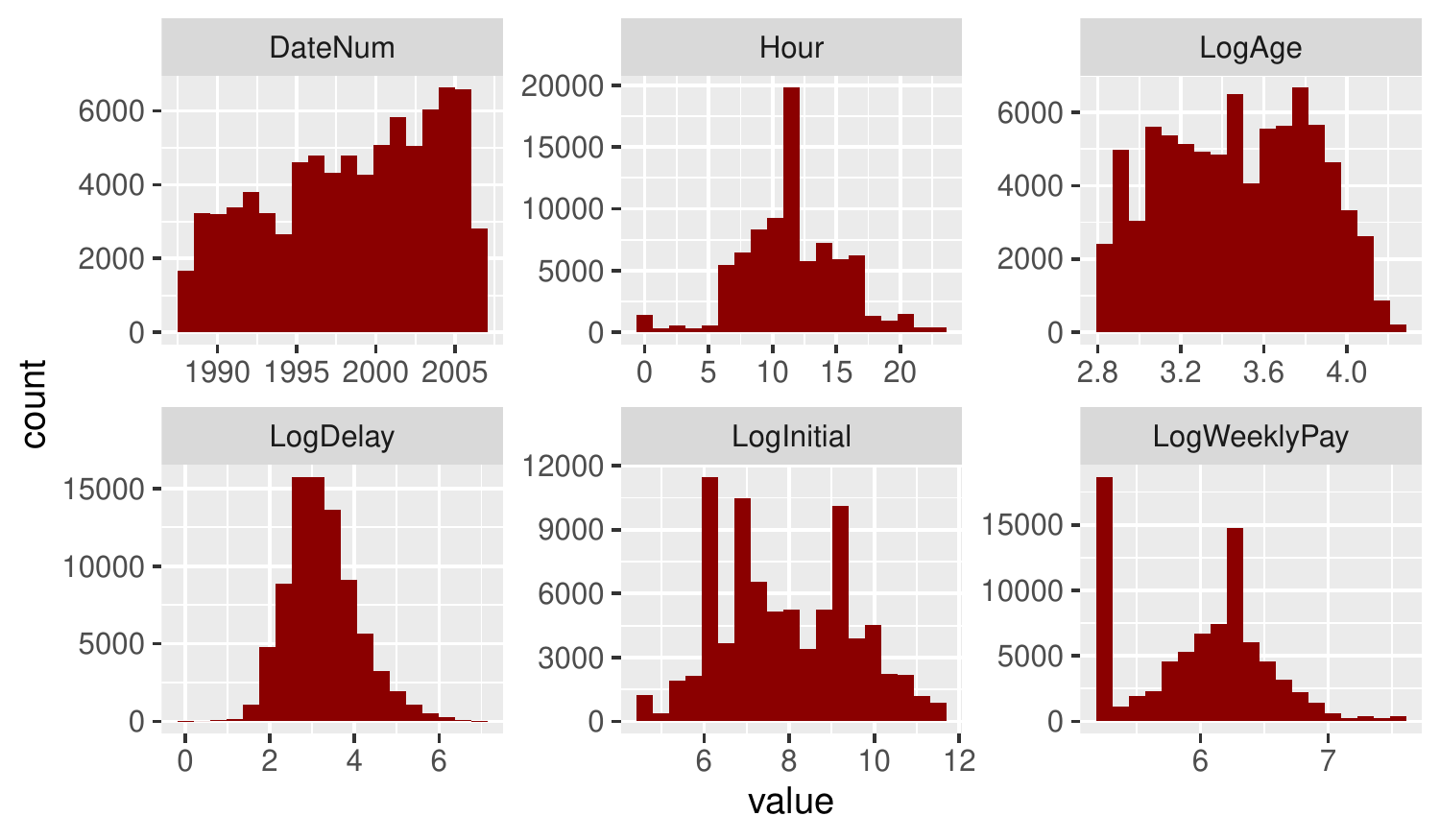}
	\caption{Histograms of continuous features.}
	\label{fig:hist}
\end{figure}
\begin{figure}
	\centering
	\includegraphics[width=0.99\textwidth]{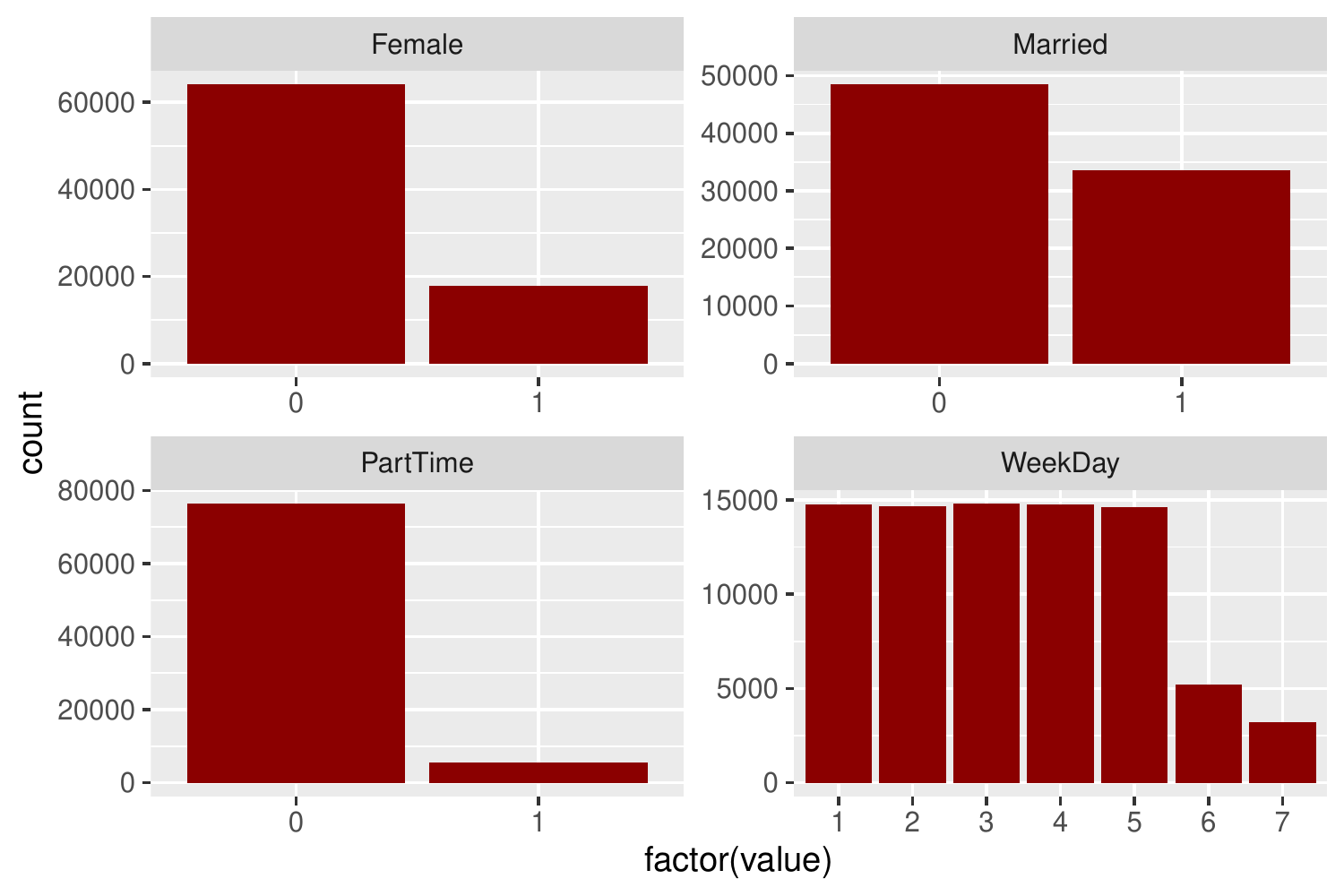}
	\caption{Barplots of discrete features.}
	\label{fig:bar}
\end{figure}
\subsection{Models}
We consider two Gamma regression models fitted with XGBoost (and using an implicit log link). The models are trained on 80\% randomly selected rows. The inspection by PDPs and SHAP dependence plots is done on $1000$ rows sampled from the training data. In order to better compare the two types of plots, and in line with Proposition~\ref{prop}, we evaluate the PDPs at all unique feature values of the subsample and shift the SHAP values vertically such that they are placed slightly below the PDP. The plots are shown on the log link scale. 
\subsubsection{Additive XGBoost model}
The first model represents each feature additively, mimicking a generalized additive Gamma regression without interactions. This is enforced by setting the tree depth to one ({\ttfamily max\_depth = 1}). The other parameters were chosen by random grid search using five-fold cross-validation on the training data. The number of boosting rounds ($484$) was found by early stopping.

Figure~\ref{fig:add} indicates that the SHAP dependence plots are indeed parallel to the PDPs. Thus, we can interpret a SHAP dependence plot like a PDP: For instance, we can say that the effect of being female increases the log expected ultimate by about $8.85 - 8.53 = 0.32$, everything else being fixed. By the absence of interaction effects, this serves as a complete description of the effects.
\begin{figure}
	\centering
	\includegraphics[width=0.98\textwidth]{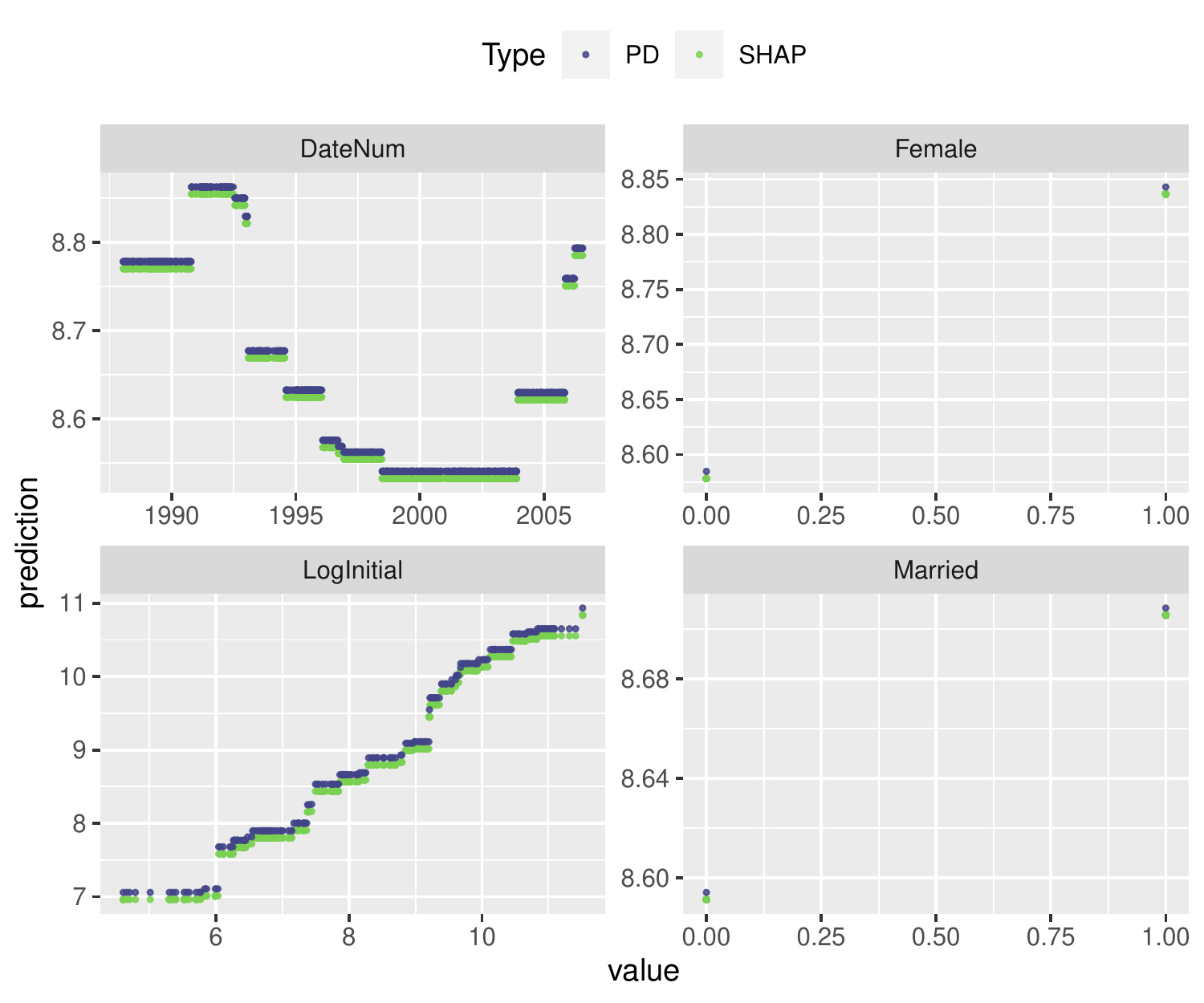}
	\caption{PDPs and (shifted) SHAP dependence plots for selected features of the additive model. The two types of plots are parallel.}
	\label{fig:add}
\end{figure}
\subsubsection{Partly additive XGBoost model}
The second model is constructed such that \texttt{DateNum} and \texttt{Female} receive additive effects, while letting other feature pairs interact with each other. 

This structure is achieved by setting \texttt{max\_depth = 2} and using interaction constraints \texttt{list(c(0, 1, 2, 3, 5, 7, 8, 9), 4, 6)} (the 0-based indices 4 and 6 represent the two additive features). Further parameters were found by a similar strategy as in the first model. The partly additive model is built from $202$ trees. Figure~\ref{fig:part} shows that the two additively modeled features yield SHAP dependence plots perfectly parallel to the PDPs. For the other features, the SHAP dependence plots show vertical scatter, indicating interaction effects.
\begin{figure}
	\centering
	\includegraphics[width=0.98\textwidth]{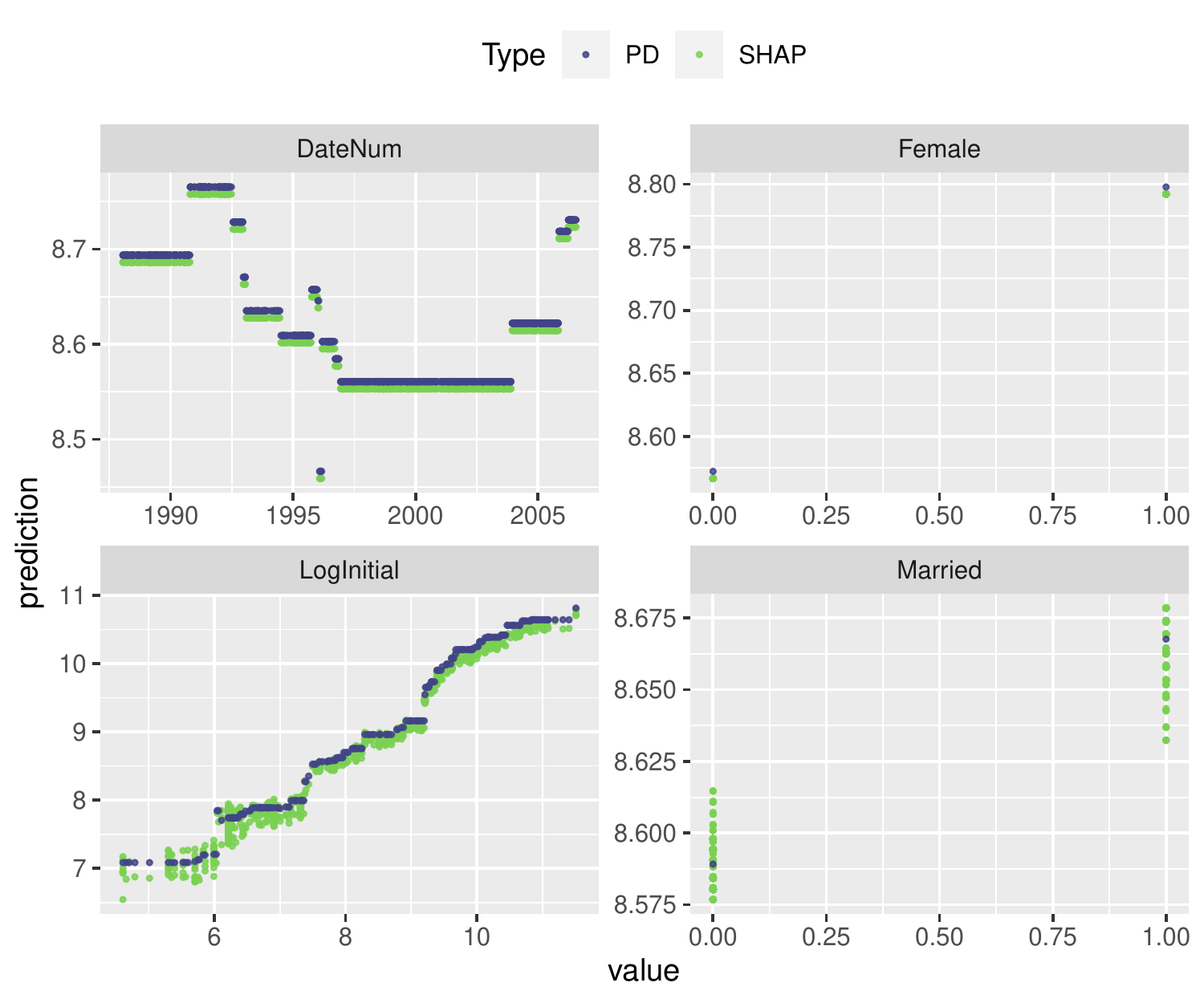}
	\caption{PDPs and (shifted) SHAP dependence plots of selected features for the partly additive model. The two types of plots are parallel only for \texttt{DateNum} and \texttt{Female}, \ie, the additively modeled covariates.}
	\label{fig:part}
\end{figure}
\section{Discussion}
We have investigated the question how to interpret SHAP dependence plots of features modeled additively in a boosted trees model. Thanks to the correspondence with PDPs, they can be interpreted ceteris paribus. For such variables, one can examine either SHAP dependence plots or PDPs, whichever is more appropriate. For example, if the model has an additive treatment effect and involves interactions between other covariates, then one can study SHAP dependence plots to see the ceteris paribus effect of treatment and still get information on interaction effects.
\bibliographystyle{plain}
\bibliography{biblio}

\begin{thebibliography}{10}

\bibitem{aas2021}
Kjersti Aas, Martin Jullum, and Anders Løland.
\newblock Explaining individual predictions when features are dependent: More
  accurate approximations to shapley values.
\newblock {\em Artificial Intelligence}, 298:103502, 2021.

\bibitem{chen2016}
Tianqi Chen and Carlos Guestrin.
\newblock Xgboost: A scalable tree boosting system.
\newblock In {\em Proceedings of the 22nd ACM SIGKDD International Conference
  on Knowledge Discovery and Data Mining}, KDD '16, pages 785--794, New York,
  NY, USA, 2016. Association for Computing Machinery.

\bibitem{friedman2001}
Jerome~H. Friedman.
\newblock Greedy function approximation: A gradient boosting machine.
\newblock {\em Ann. Stat.}, 29(5):1189--1232, 2001.

\bibitem{goldstein2015}
Alex Goldstein, Adam Kapelner, Justin Bleich, and Emil Pitkin.
\newblock Peeking inside the black box: Visualizing statistical learning with
  plots of individual conditional expectation.
\newblock {\em Journal of Computational and Graphical Statistics},
  24(1):44--65, 2015.

\bibitem{hug2019}
Nicolas Hug.
\newblock {Comment in github issue
  \url{https://github.com/christophM/interpretable-ml-book/issues/142}}, 2019.
\newblock accessed 2022-07-25.

\bibitem{ke2017}
Guolin Ke, Qi~Meng, Thomas Finley, Taifeng Wang, Wei Chen, Weidong Ma, Qiwei
  Ye, and Tie-Yan Liu.
\newblock Lightgbm: A highly efficient gradient boosting decision tree.
\newblock In {\em Advances in Neural Information Processing Systems},
  volume~30, pages 3149--3157. Curran Associates, Inc., 2017.

\bibitem{lee2015}
Simon C.~K. Lee, Sheldon Lin, and Katrien Antonio.
\newblock Delta boosting machine and its application in actuarial modeling.
\newblock Institute of Actuaries of Australia, 2015.

\bibitem{lipovetsky2001}
Stan Lipovetsky and Michael Conklin.
\newblock Analysis of regression in game theory approach.
\newblock {\em Applied Stochastic Models in Business and Industry},
  17(4):319--330, 2001.

\bibitem{lou2012}
Yin Lou, Rich Caruana, and Johannes Gehrke.
\newblock Intelligible models for classification and regression.
\newblock In {\em Proceedings of the 18th ACM SIGKDD International Conference
  on Knowledge Discovery and Data Mining}, KDD '12, pages 150--158, New York,
  NY, USA, 2012. Association for Computing Machinery.

\bibitem{lundberg2020}
Scott~M. Lundberg, Gabriel Erion, Hugh Chen, Alex DeGrave, Jordan~M. Prutkin,
  Bala Nair, Ronit Katz, Jonathan Himmelfarb, Nisha Bansal, and Su-In Lee.
\newblock From local explanations to global understanding with explainable ai
  for trees.
\newblock {\em Nature Machine Intelligence}, 2(1):2522--5839, 2020.

\bibitem{lundberg2017}
Scott~M. Lundberg and Su-In Lee.
\newblock A unified approach to interpreting model predictions.
\newblock In I.~Guyon, U.~V. Luxburg, S.~Bengio, H.~Wallach, R.~Fergus,
  S.~Vishwanathan, and R.~Garnett, editors, {\em Advances in Neural Information
  Processing Systems 30}, pages 4765--4774. Curran Associates, Inc., 2017.

\bibitem{mayer2022}
Michael Mayer, Steven~C. Bourassa, Martin Hoesli, and Donato Scognamiglio.
\newblock Machine learning applications to land and structure valuation.
\newblock {\em Journal of Risk and Financial Management}, 15(5), 2022.

\bibitem{molnar2019}
Christoph Molnar.
\newblock {\em Interpretable Machine Learning}.
\newblock \url{https://christophm.github.io/interpretable-ml-book}, 2019.

\bibitem{nori2019}
Harsha Nori, Samuel Jenkins, Paul Koch, and Rich Caruana.
\newblock Interpretml: {A} unified framework for machine learning
  interpretability.
\newblock {\em CoRR}, abs/1909.09223, 2019.

\bibitem{prokhorenkova2018}
Liudmila Prokhorenkova, Gleb Gusev, Aleksandr Vorobev, Anna~Veronika Dorogush,
  and Andrey Gulin.
\newblock Catboost: Unbiased boosting with categorical features.
\newblock In {\em Proceedings of the 32nd International Conference on Neural
  Information Processing Systems}, NIPS'18, pages 6639--6649, Red Hook, NY,
  USA, 2018. Curran Associates Inc.

\bibitem{r}
{R Core Team}.
\newblock {\em R: A Language and Environment for Statistical Computing}.
\newblock R Foundation for Statistical Computing, Vienna, Austria, 2022.

\bibitem{shapley1953}
Lloyd~S. Shapley.
\newblock A value for n-person games.
\newblock In Harold~William Kuhn and Albert~William Tucker, editors, {\em
  Contributions to the Theory of Games ({AM}-28), Volume {II}}, pages 307--318.
  Princeton University Press, dec 1953.

\bibitem{strumbelj2010}
Erik \v{S}trumbelj and Igor Kononenko.
\newblock An efficient explanation of individual classifications using game
  theory.
\newblock {\em J. Mach. Learn. Res.}, 11:1--18, 2010.

\bibitem{strumbelj2014}
Erik \v{S}trumbelj and Igor Kononenko.
\newblock Explaining prediction models and individual predictions with feature
  contributions.
\newblock {\em Knowl. Inf. Syst.}, 41(3):647–665, dec 2014.

\end{thebibliography}
\appendix
\setcounter{secnumdepth}{0}
\section{Appendix: R Code}
The following R code completely reproduces the results of Section~\ref{sec:illustration}.
\begin{Verbatim}[fontsize=\small]
# R 4.2.0
library(OpenML)    # 1.10
library(farff)     # 1.1.1
library(tidyverse) # 1.3.1
library(lubridate) # 1.8.0
library(withr)     # 2.5.0
library(xgboost)   # 1.6.0.1

#=================================================================
# Download and save dataset
#=================================================================

main <- file.path("r", "workers_compensation")

raw_file <- file.path(main, "raw.rds")
if (!file.exists(raw_file)) {
  if (!dir.exists(main)) {
    dir.create(main)
  }
  raw <- tibble(getOMLDataSet(data.id = 42876)$data)
  saveRDS(raw, file = raw_file)
} else {
  raw <- readRDS(raw_file)
}

#=================================================================
# Preprocessing
#=================================================================

# Clip small and/or large values
clip <- function(x, low = -Inf, high = Inf) {
  pmax(pmin(x, high), low)  
}

prep <- raw %>% 
  filter(WeeklyPay >= 200, HoursWorkedPerWeek >= 20) %>% 
  mutate(
    Ultimate = clip(UltimateIncurredClaimCost, high = 1e6),
    LogInitial = log(clip(InitialCaseEstimate, 1e2, 1e5)),
    LogWeeklyPay = log(clip(WeeklyPay, 100, 2000)),
    LogAge = log(clip(Age, 17, 70)),
    Female = (Gender == "F") * 1L,
    Married = (MaritalStatus == "M") * 1L,
    PartTime = (PartTimeFullTime == "P") * 1L,
    DateTimeOfAccident = as_datetime(DateTimeOfAccident),
    LogDelay = log1p(
      as.numeric(ymd(DateReported) - as_date(DateTimeOfAccident))
    ),
    DateNum = decimal_date(DateTimeOfAccident),
    WeekDay = wday(DateTimeOfAccident, week_start = 1),
    Hour = hour(DateTimeOfAccident)
  )

# Lost claim amount -> in practice, use a correction factor (0.01)
1 - with(prep, sum(Ultimate) / sum(UltimateIncurredClaimCost)) 

#=================================================================
# Variable groups
#=================================================================

y_var <- "Ultimate"
x_continuous <- c("LogInitial", "LogWeeklyPay", "LogDelay", 
                  "LogAge", "DateNum", "Hour")
x_discrete <- c("PartTime", "Female", "Married", "WeekDay")
x_vars <- c(x_continuous, x_discrete)

#=================================================================
# Univariate analysis
#=================================================================

# Some rows
prep %>%
  select(all_of(c(y_var, x_vars))) %>%
  head()

# Histogram of the response on log scale
ggplot(prep, aes(.data[[y_var]])) +
  geom_histogram(bins = 19, fill = "darkred") +
  scale_x_log10()
# ggsave("response.pdf", width = 5, height = 3)

# Histograms of continuous predictors
prep %>% 
  select(all_of(x_continuous)) %>% 
  pivot_longer(everything()) %>% 
  ggplot(aes(value)) +
  geom_histogram(bins = 19, fill = "darkred") +
  facet_wrap(~ name, scales = "free")
# ggsave("histograms.pdf", width = 6, height = 3.5)

# Barplots of discrete predictors
prep %>% 
  select(all_of(x_discrete)) %>% 
  pivot_longer(everything()) %>% 
  ggplot(aes(factor(value))) +
  geom_bar(fill = "darkred") +
  facet_wrap(~ name, scales = "free")
# ggsave("barplots.pdf", width = 6, height = 4)

#=================================================================
# Data splits for models
#=================================================================

with_seed(656, 
  .in <- sample(nrow(prep), 0.8 * nrow(prep), replace = FALSE)
)

train <- prep[.in, ]
test <- prep[-.in, ]
y_train <- train[[y_var]]
y_test <- test[[y_var]]
X_train <- prep[.in, x_vars]
X_test <- prep[-.in, x_vars]

#=================================================================
# Functions
#=================================================================

# Model agnostic PDP, not the fast one for trees
partial_dependence <- function(model, v, data) {
  grid <- unique(data[, v])
  pd <- numeric(length(grid))
  for (i in seq_along(grid)) {
    data[, v] <- grid[i]
    pd[i] <- mean(
      predict(model, data.matrix(data), outputmargin = TRUE)
    )
  }
  setNames(data.frame(grid, pd), c("value", "prediction"))
}

shap <- function(model, v, data) {
  s <- predict(model, data.matrix(data), predcontrib = TRUE)
  data.frame(value = data[, v], prediction = s[, v])
}

pdp_shap <- function(model, v, data) {
  pd <- partial_dependence(model, v, data)
  shp <- shap(model, v, data)
  eps <- diff(range(pd$prediction)) / 40
  
  # Shift SHAP values a bit (otherwise they would overlap PD)
  m <- pd[1, "value"]
  shift_shp <- pd[pd$value == m, "prediction"] - 
    mean(shp[shp$value == m, "prediction"])
  shp <- transform(shp, prediction = prediction + shift_shp - eps)
  dat <- bind_rows(list(PD = pd, SHAP = shp), .id = "Type")
  cbind(dat[order(dat$Type, decreasing = TRUE), ], Feature = v)
}

pdp_shap_multi <- function(model, vs, data) {
  dat <- bind_rows(lapply(vs, function(v) pdp_shap(model, v, data)))
  ggplot(dat, aes(value, prediction, Type, color = Type)) +
    geom_point(alpha = 0.8, size = 0.6) +
    facet_wrap(reformulate("Feature"), scales = "free") +
    scale_color_viridis_d(begin = 0.2, end = 0.8) +
    theme(legend.position = "top")
}

#=================================================================
# XGBoost: Additive model 
#=================================================================

nrounds <- 484

params <- list(
  learning_rate = 0.1, 
  objective = "reg:gamma", 
  max_depth = 1, 
  colsample_bynode = 1, 
  subsample = 1, 
  reg_alpha = 3, 
  reg_lambda = 1, 
  tree_method = "hist", 
  min_split_loss = 0.001, 
  nthread = 7 # adapt
)

dtrain <- xgb.DMatrix(data.matrix(X_train), label = y_train)

with_seed(3089, 
  xgb_add <- xgb.train(
    params = params, data = dtrain, nrounds = nrounds
  )
)

# Inspect
with_seed(536,
  X_small <- data.frame(sample_n(X_train, 1000))
)
vs <- c("Female", "DateNum", "LogInitial", "Married")

pdp_shap_multi(xgb_add, vs, X_small)
# ggsave("xgb_add.pdf", width = 6, height = 5)

#=================================================================
# XGBoost: Partly additive model
#=================================================================

# Build interaction constraint vector
additive <- c("Female", "DateNum")
ic <- c(
  list(which(!(x_vars %in% additive)) - 1),
  as.list(which(x_vars %in% additive) - 1)
)
ic

nrounds2 <- 202

params2 <- list(
  learning_rate = 0.1, 
  objective = "reg:gamma", 
  max_depth = 2, 
  interaction_constraints = ic,
  colsample_bynode = 0.8, 
  subsample = 1, 
  reg_alpha = 0, 
  reg_lambda = 2, 
  min_split_loss = 0, 
  tree_method = "hist", 
  nthread = 7
)

with_seed(839, 
  xgb_part_add <- xgb.train(
    params = params2, data = dtrain, nrounds = nrounds2
  )
)

pdp_shap_multi(xgb_part_add, vs, X_small)
# ggsave("xgb_part_add.pdf", width = 6, height = 5)
\end{Verbatim}
\end{document}